\documentclass[11pt]{article}
\usepackage[utf8]{inputenc}
\usepackage[margin=1in]{geometry}
\title{Near-Optimal Last-Iterate Convergence of SGD
in the Smooth Interpolation Regime}

\usepackage{amsmath,amsfonts,amsthm,amssymb}
\usepackage{mathtools}
\usepackage{algorithm,algorithmic}
\usepackage{thmtools}
\usepackage{makecell}
\usepackage{multirow}        %
\usepackage{booktabs}        %
\usepackage{ifthen}

\newcommand{\Otil}{\widetilde{O}}

\usepackage{markdown}
\usepackage{xcolor}
\usepackage{enumitem}

\usepackage{makecell}       %
\usepackage{multirow}       %
\usepackage{caption}        %
\usepackage{float} 
\usepackage{booktabs, multirow, makecell, array, xspace}
\usepackage{natbib}
\newcolumntype{C}[1]{>{\centering\arraybackslash}m{#1}}

\renewcommand{\arraystretch}{1.2}

\newcommand{\R}{\mathbb{R}}

\newcommand{\E}{\boldsymbol{\mathbb{E}}}
\newcommand{\norm}[2][a]{\brackets[#1]{\|}{#2}{\|}}

\newcommand{\brackets}[4][a]{%
  \ifx a#1\left#2 #3 \right#4\else
  \ifx s#1#2 #3 #4\else
  \ifx b#1\big#2 #3 \big#4\else
  \ifthenelse{\equal{#1}{B}} {\Big#2 #3 \Big#4} {}
  \ifthenelse{\equal{#1}{Bg}} {\Bigg#2 #3 \Bigg#4} {}
  \fi\fi\fi
}

\newcommand{\eqq}{\coloneqq}

\DeclareMathOperator*{\argmin}{arg\,min}

\newcommand{\cZ}{\mathcal Z}

\newenvironment{aligni*}{\begin{math}}{\end{math}}
\newenvironment{equationi*}{\begin{math}}{\end{math}}

\usepackage{pifont}

\definecolor{plum}{rgb}{.6,0,.6}

\theoremstyle{plain}
\newtheorem{theorem}{Theorem}

\newtheorem{lemma}{Lemma}
\newtheorem{corollary}{Corollary}

\theoremstyle{definition}

\theoremstyle{remark}

\newtheorem*{theorem*}{\textbf{Theorem}}
\newtheorem*{lemma*}{\textbf{Lemma}}

\usepackage[extdef=true]{delimset}
\usepackage{adjustbox}
\usepackage{threeparttable}

\newcommand{\halfcmark}{\checkmark\kern-5pt\raisebox{2pt}{\scriptsize\textbackslash}\kern1pt}

\newcommand{\ip}[2]{\langle #1,#2\rangle}

\usepackage{hyperref}
\usepackage[capitalize]{cleveref}
\author{%
    Matan Schliserman%
    \thanks{\scriptsize Blavatnik School of Computer Science and AI, Tel Aviv University; \texttt{schliserman@mail.tau.ac.il}.}%
    \and
    Tomer Koren%
    \thanks{\scriptsize Blavatnik School of Computer Science and AI, Tel Aviv University, and Google Research Tel Aviv; \texttt{tkoren@tauex.tau.ac.il}.}
}

\begin{document}

\maketitle

\begin{abstract}
We study the last iterate of stochastic gradient descent (SGD) with a constant stepsize on smooth convex objectives in the low-noise and interpolation regimes, where the variance $\sigma_\star^2$ of the stochastic gradients at the optimum is small or zero.
We establish that after $T$ steps of SGD on $\beta$-smooth convex losses with any constant stepsize $0<\eta<2/\beta$, the last iterate attains expected excess risk $\smash{\widetilde O\bigl(\tfrac{1}{\eta(2-\beta\eta)^3T}+\tfrac{\eta}{(2-\beta\eta)^4}\sigma_\star^2\bigr)}$, which gives the rate $\smash{\widetilde O(1/T)}$ for every constant stepsize in the interpolation regime ($\sigma_\star=0$), matching the $\Omega(1/T)$ lower bound up to logarithmic factors and improving upon the $O(1/\sqrt T)$ rate of \citet{attia2025fast} for $\eta=1/\beta$.
Our analysis is based on a new technique, which compares the second moments of the stochastic gradients along the SGD trajectory across only logarithmically many time scales.
For quadratic losses in the interpolation regime, an adaptation of this technique gives an $O(1/T)$ rate for every constant stepsize, extending the recent optimal rate for the randomized Kaczmarz method of \citeauthor{derezinski2026} (2026) to stepsizes above the greedy stepsize with an arguably simpler proof.
\end{abstract}

\section{Introduction}
\label{sec:intro}

We study the convergence of stochastic gradient descent (SGD) for smooth convex objectives in the low-noise and interpolation regimes. Concretely, we consider optimization problems of the form
$$
    \min_{x\in\R^d}F(x)
    \eqq
    \E_{z\sim\cZ}[f(x;z)],
$$
where $\cZ$ is a distribution over a sample space $Z$, and $f(\cdot;z)\colon\R^d\to\R$ is convex and $\beta$-smooth for all $z\in Z$.
In the \emph{low-noise regime}, the gradient noise at the optimum, defined as $\sigma_\star^2\eqq\E_z\norm{\nabla f(x^\star;z)}^2$ for $x^\star\in\argmin_{x\in\R^d}F(x)$, is assumed to be small or even zero. The special case where $\sigma_\star^2=0$ is known as the \emph{interpolation regime}, in which $x^\star$ minimizes $f(\cdot;z)$ for almost all $z\in Z$.
We analyze the expected excess risk of SGD, which, given a stepsize $\eta>0$ and an i.i.d.\ sample $z_1,\ldots,z_T\sim\cZ$, performs the following updates, starting from an initialization $x_1\in\R^d$:
\begin{align*}
    x_{t+1}=x_t-\eta\nabla f(x_t;z_t).
    \qquad
    (t=1,\ldots,T)
\end{align*}
Standard convergence bounds of SGD apply to an average (possibly weighted) of the iterates. It is by now a classical result that the average iterate of SGD with a constant stepsize converges at the \emph{fast rate} $O(1/T)$ in the interpolation regime with $\eta=\Theta(1/\beta)$, and at the rate $O(1/T+\sigma_\star/\sqrt T)$ in the low-noise regime with a suitably tuned stepsize~\citep{srebro2010smoothness,ma2018power}, as compared to the slower $O(1/\sqrt T)$ rate that holds more generally.

In practice, however, SGD is commonly run with a constant stepsize, independent of $T$ and of order $\Theta(1)$, and returns its last iterate. This has motivated a long line of work on \emph{last-iterate bounds}, namely, bounds that apply to the expected excess risk $\E[F(x_T)-F(x^\star)]$ of the last SGD iterate rather than to an average of the iterates~\citep{ICML2012Rakhlin_261,shamir13sgd,jain2019making,harvey2019tight,orabona2020last,zamani2023exact,liu2024revisiting}. For convex Lipschitz losses, the last iterate can be worse than the average by a logarithmic factor~\citep{shamir13sgd,harvey2019tight}, and stepsize schedules that depend on the horizon can close this gap~\citep{jain2019making}. These results, however, require stepsizes of $O(1/\sqrt T)$, and thus do not capture the fast rates attained by the average iterate in the smooth low-noise regime.

A recent line of work has focused on the last iterate of SGD with a constant stepsize in the smooth low-noise regime. This regime is natural for modern overparameterized models, which often interpolate their training data~\citep{ma2018power}. It also captures classical procedures that output their last iterate, such as the randomized Kaczmarz method~\citep{needell2014stochastic} and continual learning of realizable tasks~\citep{evron2022catastrophic,evron2025better,derezinski2026colt,derezinski2026}, which can be presented as instances of SGD in the interpolation regime with the ``greedy'' stepsize $\eta=1/\beta$ (e.g. \cite{evron2025better}). For quadratic losses, \citet{varre2021last} proved an $O(\log T/T)$ rate with a stepsize of order $1/(\beta\log T)$, and at the greedy stepsize, the known rates improved from $O(T^{-1/4})$~\citep{evron2025better} to $O(T^{-3/4})$~\citep{derezinski2026colt} and, very recently, to the optimal $O(1/T)$~\citep{derezinski2026}; the latter two results rely on the linear structure of the updates. For general smooth convex losses, \citet{attia2025fast} and \citet{garrigos2025last} obtained fast last-iterate rates with stepsizes of order $1/(\beta\log T)$, which depend on the horizon. For larger constant stepsizes, these guarantees deteriorate polynomially in $T$, and at the greedy stepsize, the best known rate is $O(1/\sqrt T)$~\citep{attia2025fast} (see \cref{tab:rates}). Thus, for general smooth convex losses with a constant stepsize, there is a gap between the best known last-iterate bounds and the fast $O(1/T)$ rate of the average iterate.

In this work, we establish that the last iterate of SGD on convex and $\beta$-smooth losses with any constant stepsize $0<\eta<2/\beta$ attains an expected excess risk of $\smash{\Otil\bigl(\tfrac{1}{\eta T}+\eta\sigma_\star^2\bigr)}$. To achieve this bound, we introduce a new multiscale technique for last-iterate analysis, which tracks the sequence of second moments of the stochastic gradients along the SGD trajectory. Then, for quadratic losses in the interpolation regime, we show that an adaptation of our technique removes the logarithmic factors and achieves an $O(1/T)$ rate for every constant stepsize $\eta< 2/\beta$, extending the recent optimal rate of \citet{derezinski2026}, proved for $\eta\leq 1/\beta$, with an arguably simpler proof.  Finally, for completeness, we apply our bounds to the randomized Kaczmarz method and related applications.

\begin{table}[t]
\centering
\small
\renewcommand{\arraystretch}{1.6}
\setlength{\tabcolsep}{3pt}
\begin{tabular}{@{}llcc@{}}
\toprule
\textbf{Reference} & \textbf{Losses} & \textbf{\begin{tabular}[c]{@{}c@{}}Bound for\\ $0<\eta<2/\beta$\end{tabular}} & \textbf{\begin{tabular}[c]{@{}c@{}}Bound for\\ $\eta=1/\beta$, $\sigma_\star=0$\end{tabular}} \\
\midrule
\citet{evron2025better} & quadratic & $O\Bigl(\dfrac{1}{\eta \delta T^{\delta^2/4}}\Bigr)$ & $O\Bigl(\dfrac{1}{T^{1/4}}\Bigr)$ \\
\citet{attia2025fast} & convex and smooth & $\Otil\Bigl(\dfrac{1}{\eta \delta T^{\delta/2}}+\dfrac{\eta T^{1-\delta/2}\sigma_\star^2}{\delta^2}\Bigr)$ & $O\Bigl(\dfrac{1}{\sqrt T}\Bigr)$ \\
\citet{garrigos2025last} & convex and smooth & $\Otil\Bigl(\dfrac{1}{\eta(1-\beta\eta)T^{1-\phi}}+\dfrac{\eta T^{\phi}\sigma_\star^2}{(1-\beta\eta)^2}\Bigr)^{\dagger}$ & --- \\
\citet{derezinski2026} & quadratic & $O\Bigl(\dfrac1{\eta T}\Bigr)^{\ddagger}$ & $O\Bigl(\dfrac1T\Bigr)$ \\
\textbf{This work} (\cref{thm:main}) & convex and smooth & $\Otil\Bigl(\dfrac{1}{\eta \delta^3T}+\dfrac{\eta\sigma_\star^2}{\delta^4}\Bigr)$ & $\Otil\Bigl(\dfrac{1}{T}\Bigr)$ \\
\textbf{This work} (\cref{thm:quadratic}) & quadratic & $O\Bigl(\dfrac{1}{\eta \delta^3T}\Bigr)$ & $O\Bigl(\dfrac{1}{T}\Bigr)$ \\
\midrule
\citet{zhang2023lower} & convex and smooth & $\Omega\Bigl(\dfrac{1}{\eta T}\Bigr)$ & $\Omega\Bigl(\dfrac{1}{T}\Bigr)$ \\
\bottomrule
\end{tabular}
\caption{Last-iterate bounds on the expected excess risk of SGD after $T$ steps with a constant stepsize $\eta$ independent of $T$, for $\norm{x_1-x^\star}=1$, $\delta\eqq2-\beta\eta$, and $\phi\eqq2\beta\eta/(1+\beta\eta)$. The rows labeled ``quadratic'' assume interpolation ($\sigma_\star=0$). ``---'' indicates that no bound is stated. ${}^\dagger$Only for $\eta<1/\beta$. ${}^\ddagger$Only for $\eta\leq1/\beta$. The last row is a lower bound in the interpolation regime.}
\label{tab:rates}
\vspace{-10pt}
\end{table}

\subsection{Summary of contributions}

In more detail, our main contributions are as follows (see a comparison to previous work in \cref{tab:rates}):
\begin{enumerate}[label=(\roman*),leftmargin=*]
    \item In \cref{thm:main} (\cref{sec:main-result}), we establish a convergence rate of
    $
        \Otil\biggl(\frac{\norm{x_1-x^\star}^2}{\eta(2-\beta\eta)^3T}+\frac{\eta}{(2-\beta\eta)^4}\sigma_\star^2\biggr)
    $
    for the last iterate of SGD on convex, $\beta$-smooth losses, with any constant stepsize $0<\eta<2/\beta$.
    In the interpolation regime, the rate is $\Otil(1/T)$ for every such stepsize, matching the lower bound up to logarithmic factors. Compared with \citet{attia2025fast}, this replaces the decay $T^{-(1-\beta\eta/2)}$ by $\Otil(1/T)$ and replaces polynomial growth of the noise term by logarithmic growth.
    In the low noise regime, tuning the stepsize gives the near-optimal rate of  $\Otil(1/T+\sigma_\star/\sqrt T)$.

    \item On the technical side, we bound the last element of the sequence of second moments of the stochastic gradients by a multiscale comparison of the sequence with its shifted copies at only logarithmically many time scales, in contrast to the polynomially many past iterates used in previous analyses.  This is achieved using virtual iterates, namely, SGD trajectories with a block of consecutive updates deleted, which enable us to relate the risk of the last iterate to the average risk of the iterates with only logarithmic factors in $T$, instead of the polynomial factors of previous bounds. This proof technique is outlined in \cref{sec:outline}.

    \item In \cref{thm:quadratic} in \cref{sec:quadratic}, we show that our technique gives the rate $O\bigl(\frac{\|x_1-x^\star\|^2}{\eta(2-\beta\eta)^3T}\bigr)$ for quadratic losses in the interpolation regime, for every constant stepsize $0<\eta<2/\beta$. This extends the optimal $O(1/T)$ rate of \citet{derezinski2026} to stepsizes above the greedy stepsize, with an arguably simpler proof. 

 \item For completeness, in \cref{app:applications} we apply our results to random projections onto convex sets, the randomized Kaczmarz method, and forgetting in continual learning, obtaining, for instance, $O(1/T)$ rates for the randomized Kaczmarz method and for forgetting in continual linear regression.

\end{enumerate}

\paragraph{Open problems.}
Our upper bound exceeds the $\Omega(1/T)$ lower bound by a factor of $\log^3T$. The main question left open is whether these logarithmic factors are necessary for general smooth convex losses. This question is of particular interest since, for non-smooth convex losses, logarithmic factors are necessary when the stepsizes do not depend on the horizon~\citep{harvey2019tight,kornowski2026}. However, the known constructions rely heavily on nonsmoothness, and it remains unclear whether logarithmic factors are also necessary in the smooth case. In addition, we leave as an open problem the exact dependence of the bounds on $2-\beta\eta$.

\subsection{Additional related work}
\label{sec:related}

\paragraph{Last-iterate analysis of SGD.}
Last-iterate convergence of SGD has been studied extensively for convex Lipschitz losses in the gradient oracle model, where decaying or horizon-dependent stepsizes yield near-optimal guarantees~\citep{ICML2012Rakhlin_261,shamir13sgd,harvey2019tight,jain2019making,orabona2020last,zamani2023exact,liu2024revisiting}.
\citet{cortild2025new} used a computer-aided Lyapunov analysis to prove average-iterate guarantees.
For least squares under interpolation, \citet{varre2021last} obtained constant-stepsize guarantees under additional assumptions on the features and solution.
Our analysis gives a $\Otil(1/T)$ last-iterate rate in the interpolation regime for every convex and smooth loss and every constant stepsize below $2/\beta$.

\paragraph{SGD in the interpolation regime.}
Interpolation is a natural regime for overparameterized deep networks, which are commonly trained with SGD until they fit their training data. A growing literature studies SGD for smooth objectives~\citep[e.g.,][]{lan2012optimal,bach2013non,ge2019stepDecay}, and in particular in overparameterized and interpolation settings~\citep[e.g.,][]{vaswani2019fast,berthier2020tight,varre2021last,wu2022last,liu2023aiming}, mostly for underdetermined linear regression or for smooth convex component losses with a strongly convex population objective. In contrast, our results impose neither a parametric form on the losses nor a strongly convex objective.

\paragraph{Stability and coupling arguments.}
Our last-iterate analysis involves virtual iterates, namely, coupled copies of the SGD trajectory, a technique that has appeared in other contexts. Algorithmic stability analyses compare SGD trajectories that differ in a few samples and exploit the nonexpansiveness of smooth convex gradient updates~\citep{hardt2016train,lei20stab,schliserman2022stability}. Our comparison deletes blocks of updates. Under i.i.d.\ sampling, the resulting virtual trajectories have the distributions of earlier SGD iterates, turning deletion stability into control of variation across time. \citet{sherman2021optimal} and \citet{liu2026random} also used coupled virtual sequences, constructed by swapping two entries of a random permutation, to study stability under sampling without replacement.

\paragraph{The Kaczmarz method.}
The Kaczmarz method solves consistent linear systems by successive projections onto the solution sets of single equations or blocks of equations~\citep{karczmarz1937angenaherte,elfving1980block,needell2014paved}, and more general projection methods find a common point of convex sets~\citep{gubin1967methodProjections,boyd2003alternating}. In \cref{app:projections}, we apply our bounds to the randomized Kaczmarz method and obtain guarantees that are independent of the condition number. Classical analyses establish linear convergence at rates that depend on the condition number~\citep[e.g.,][]{strohmer2009randomized,gower2015randomized,needell2014paved,oswald2015convergence}, which can be arbitrarily large, while recent work obtained tight bounds for the cyclic method~\citep{yu2026tight} and condition-number-free rates~\citep{derezinski2026colt,derezinski2026}, including the optimal $O(1/T)$ rate, which our bounds recover through a more general and arguably simpler analysis.

\paragraph{Continual learning.}
Continual learning trains a model on a sequence of tasks, during which its performance on earlier tasks may degrade, a phenomenon known as catastrophic forgetting. In \cref{app:continual}, we show that our bounds apply to forgetting in continual linear regression and classification, settings whose theory has been studied extensively~\citep[e.g.,][]{doan2021NTKoverlap,lin2023theory,peng2023ideal,goldfarb2023analysis,li2023fixed,hiratani2024disentangling,evron23continualClassification,levinstein2025optimal}. The extent of forgetting depends on task similarity, overparameterization, and the trade-off between stability and plasticity~\citep{mermillod2013stability,goldfarb2024theJointEffect}, and it can diminish under cyclic or random task orderings~\citep{lesort2022scaling,hemati2024continual,swartworth2023nearly,jung2025convergence,cai2025lastIterate}. Beyond linear models, the theory has recently been extended to deep networks~\citep{schliserman2026convergence}.

\section{Setup and main result}
\label{sec:main}

\paragraph{Problem setup.}
Let $Z$ be an index set for convex loss functions $f(\cdot;z)\colon\R^d\to\R$, and $\cZ$ be an arbitrary distribution over $Z$.
We consider the unconstrained stochastic optimization problem
\begin{align*}
    \min_{x\in\R^d}\Bigl\{F(x)\eqq\E_{z\sim\cZ}f(x;z)\Bigr\},
\end{align*}
under the assumption that the loss functions are individually smooth.
Specifically, we assume that for all $z\in Z$, the function $f(\cdot;z)$ is $\beta$-smooth.\footnote{That is, $\norm{\nabla f(y;z)-\nabla f(x;z)}\leq\beta\norm{y-x}$ for all $x,y\in\R^d$.} (Throughout, $\norm{\cdot}$ refers to the $\ell_2$ norm, and $\log$ to the natural logarithm.)
We further assume that the objective $F$ admits a minimizer, denoted by $x^\star\in\argmin_{x\in\R^d}F(x)$, and that the variance of the stochastic gradients at $x^\star$ is bounded, $\sigma_\star^2\eqq\E_{z\sim\cZ}\norm{\nabla f(x^\star;z)}^2<\infty$.
The special case where $\sigma_\star^2=0$ is referred to as the \emph{interpolation regime}, in which $x^\star$ minimizes $f(\cdot;z)$ for almost all $z\in Z$.

Throughout the paper, we consider the fixed stepsize SGD algorithm, which given an initialization $x_1\in\R^d$, a stepsize $\eta>0$ and an i.i.d.\ sample $z_1,z_2,\ldots\sim\cZ$, performs at each step $t=1,2,\ldots$ the update
\begin{align}\label{eq:sgd}
    x_{t+1}=x_t-\eta\nabla f_t(x_t),
    \qquad
    \text{where } f_t(x)\eqq f(x;z_t),
\end{align}
and returns $x_T$. We bound the expected excess risk $\E[F(x_T)-F(x^\star)]$ of the last iterate.

\subsection{Main result}
\label{sec:main-result}

Our main result shows that the expected excess risk of the last iterate decays at the rate $1/T$, up to logarithmic factors and a noise term, for every constant stepsize $0<\eta<2/\beta$.

\begin{theorem}
\label{thm:main}
Let $f\colon\R^d\times Z\to\R$ be such that $f(\cdot;z)$ is $\beta$-smooth and convex for every $z\in Z$. Assume there exists a minimizer $x^\star\in\argmin_{x\in\R^d}F(x)$, and that $\sigma_\star^2<\infty$.
Then, for SGD initialized at $x_1\in\R^d$ with stepsize $0<\eta<2/\beta$, we have the following last-iterate guarantee for every $T\geq1$:
\begin{align*}
    \E[F(x_T)-F(x^\star)]
    \leq
    1056\biggl(
    \frac{\norm{x_1-x^\star}^2}{\eta(2-\beta\eta)^3T}
    +\frac{\eta}{(2-\beta\eta)^4}\sigma_\star^2
    \biggr)\log^3(eT).
\end{align*}
In particular:
\begin{enumerate}[nosep,label=(\roman*),leftmargin=*]
    \item in the interpolation regime ($\sigma_\star=0$), when $\eta=\frac1\beta$, it holds that
    \begin{align*}
        \E[F(x_T)-F(x^\star)]
        \leq
        \frac{1056\beta\norm{x_1-x^\star}^2\log^3(eT)}{T};
    \end{align*}
    \item when $\eta=\min\bigl\{\frac1\beta,\frac{\norm{x_1-x^\star}}{\sigma_\star\sqrt T}\bigr\}$, it holds that
    \begin{align*}
        \E[F(x_T)-F(x^\star)]
        \leq
        1056\biggl(\frac{\beta\norm{x_1-x^\star}^2}{T}+\frac{2\sigma_\star\norm{x_1-x^\star}}{\sqrt T}\biggr)\log^3(eT).
    \end{align*}
\end{enumerate}
\end{theorem}

In the interpolation regime ($\sigma_\star=0$), the bound gives the rate $\Otil(1/T)$ for every constant stepsize $0<\eta<2/\beta$, whereas the bound of \citet{attia2025fast} decays only as $T^{-(1-\beta\eta/2)}$. In the low-noise regime, the stepsize in (ii) is tuned to the horizon and gives the near-optimal rate $\Otil(1/T+\sigma_\star/\sqrt T)$. 

\section{Proof of main result}
\label{sec:outline}

In this section we prove \cref{thm:main}. For this, our main techincal tool is to analyze the sequence of the second moments of the stochastic gradients along the SGD dynamics, defined as
\begin{align}\label{eq:sigma}
    \sigma_s
    \eqq
    \bigl(\E\norm{\nabla f_s(x_s)}^2\bigr)^{1/2}
    \qquad
    (s\geq1).
\end{align}
Below we elaborate on the main steps of the proof.

\subsection{A multiscale comparison for the last element of a sequence}
\label{sec:seq-last}

The first step is to bound the last element $\sigma_t$ of the sequence by the average of the sequence and by a multiscale comparison of the sequence with itself at the lags $h=1,2,4,\ldots$, all with the same starting point $r$, which will be determined later (in \cref{sec:seq-gradient}). In fact, we show the following lemma, which holds for any real sequence.
\begin{lemma}
\label{lem:endpoint}
Let $t\geq1$, let $J\geq0$ be the integer for which $2^J\leq t+1-\lceil t/4\rceil<2^{J+1}$, and let $1\leq r\leq\lceil t/4\rceil$. Then, for all real numbers $\sigma_1,\ldots,\sigma_t$,
\begin{align*}
    \sigma_t^2
    \leq
    \frac4t\sum_{s=1}^t\sigma_s^2+2\log^2t\max_{0\leq j<J}\frac1{2^j}\sum_{s=r}^{t-2^j}(\sigma_{s+2^j}-\sigma_s)^2,
\end{align*}
where the maximum is $0$ if $J=0$.
\end{lemma}
The bound of \cref{lem:endpoint} involves only $J\le\log_2t$ comparisons, one for each lag $2^j$. These lags correspond to different time scales, from consecutive elements of the sequence ($j=0$) to elements about $t/2$ steps apart ($j=J-1$). The normalization by $2^j$ makes these costs comparable across the scales, and the factor $\log^2t$ comes from the number of scales. In contrast, previous last-iterate analyses~\citep{shamir13sgd,zamani2023exact,liu2024revisiting,attia2025fast} compare the last iterate with polynomially many past iterates, which results in additional polynomial factors compared to our analysis.

\begin{proof}[Proof of \cref{lem:endpoint}]
Fix $t$, $J$ and $r$ as in the statement, and real numbers $\sigma_1,\ldots,\sigma_t$. Since $1\leq\lceil t/4\rceil<t/4+1$, we have $2^J\leq t$ and $2^J>(t+1-\lceil t/4\rceil)/2>3t/8$, and as $r\leq\lceil t/4\rceil$, also $2^J\leq t+1-r$. We compare $\sigma_t$ with the average of the last $2^J$ elements. For $j\in\{0,\ldots,J\}$, let $M_j\eqq\frac1{2^j}\sum_{s=t-2^j+1}^t\sigma_s$, so that $M_0=\sigma_t$. For $j\in\{0,\ldots,J-1\}$ and $h=2^j$, $t-2h+1\geq r$, and since $M_j-M_{j+1}=\frac1{2h}\sum_{s=t-2h+1}^{t-h}(\sigma_{s+h}-\sigma_s)$, by Jensen inequality,
\begin{align*}
    &(M_j-M_{j+1})^2
    \leq
    \frac1{4h}\sum_{s=t-2h+1}^{t-h}(\sigma_{s+h}-\sigma_s)^2\leq
    \frac1{4h}\sum_{s=r}^{t-h}(\sigma_{s+h}-\sigma_s)^2. 
    \\&M_J^2
    \leq
    \frac1{2^J}\sum_{s=t-2^J+1}^t\sigma_s^2
    \leq
    \frac8{3t}\sum_{s=1}^t\sigma_s^2.
\end{align*}
Finally, $\sigma_t=M_J+\sum_{j=0}^{J-1}(M_j-M_{j+1})$, thus, by young's inequality and Cauchy--Schwarz,
\begin{align*}
    \sigma_t^2
   \leq
    \frac32M_J^2+3J\sum_{j=0}^{J-1}(M_j-M_{j+1})^2
    \leq
    \frac4t\sum_{s=1}^t\sigma_s^2+\frac{3J^2}4\max_{0\leq j<J}\frac1{2^j}\sum_{s=r}^{t-2^j}(\sigma_{s+2^j}-\sigma_s)^2.
\qedhere
\end{align*}
\end{proof}

\subsection{The comparison at a single scale}
\label{sec:seq-comparison}

The next step is to bound the comparison at a single lag $h$, namely $\sum_{s=r}^{t-h}(\sigma_{s+h}-\sigma_s)^2$, which appears in \cref{lem:endpoint}. The sum that defines this distance begins at the starting point $r$, which is common to all the lags. The following lemma bounds this distance by the expected squared distance between the iterates $x_r$ and $x_{r+h}$.
\begin{lemma}
\label{lem:comparison}
Under the conditions of \cref{thm:main}, let $\sigma_s$ be as defined in \cref{eq:sigma}. Then, for all integers $r,h\geq1$ and $t\geq r+h$,
\begin{align*}
    \sum_{s=r}^{t-h}(\sigma_{s+h}-\sigma_s)^2
    \leq
    \frac\beta{\eta(2-\beta\eta)}\E\norm{x_{r+h}-x_r}^2.
\end{align*}
\end{lemma}

The proof of \cref{lem:comparison} compares SGD with \emph{virtual iterates}, obtained by deleting a block of updates. For integers $r,h\geq1$, the virtual iterates $\tilde x_s$, $s\geq r+h$, restart SGD from the past iterate $x_r$ at time $r+h$, namely,
\begin{align}\label{eq:restart}
    \tilde x_s
    \eqq
    \begin{cases}
        x_r, & s=r+h,\\
        \tilde x_{s-1}-\eta\nabla f_{s-1}(\tilde x_{s-1}), & s>r+h;
    \end{cases}
\end{align}
that is, $\tilde x$ is obtained from SGD by deleting the updates at steps $r,\ldots,r+h-1$, and from time $r+h$ on, it is updated with the same samples as SGD. The virtual iterates have two properties. First, in distribution, they are past iterates of SGD: since the samples are i.i.d., the pair $(\tilde x_s,z_s)$ is distributed as the pair $(x_{s-h},z_{s-h})$, and in particular, $\E\norm{\nabla f_s(\tilde x_s)}^2=\sigma_{s-h}^2$. Second, they are coupled with the current iterates of SGD: the virtual iterates and SGD are updated with the same samples, and as gradient steps with a common sample bring two points closer to each other, the virtual iterates remain close to the iterates of SGD. The first property expresses $\sigma_s$ through the virtual iterate $\tilde x_{s+h}$, and the second relates $\tilde x_{s+h}$ to the SGD iterate $x_{s+h}$, whose stochastic gradient has second moment $\sigma_{s+h}$.
Below we prove \cref{lem:comparison} by combining these two properties.
\begin{proof}[Proof of \cref{lem:comparison}]
Fix $r$, $h$ and $t$ as in the statement, let $\tilde x_s$, $s\geq r+h$, be the virtual iterates defined in \cref{eq:restart}, and for $s\geq r+h$, let $u_s\eqq\nabla f_s(x_s)$ and $\tilde u_s\eqq\nabla f_s(\tilde x_s)$. We prove the two inequalities
\begin{align}\label{eq:comparison-chain}
    \sum_{s=r}^{t-h}(\sigma_{s+h}-\sigma_s)^2
    \leq
    \sum_{s=r+h}^t\E\norm{u_s-\tilde u_s}^2
    \leq
    \frac\beta{\eta(2-\beta\eta)}\E\norm{x_{r+h}-\tilde x_{r+h}}^2,
\end{align}
which give the claim, as $\tilde x_{r+h}=x_r$. For the first inequality, we use the first property: $\E\norm{\tilde u_s}^2=\E\norm{\nabla f_{s-h}(x_{s-h})}^2=\sigma_{s-h}^2$ for every $s\geq r+h$, and thus, by the Cauchy--Schwarz inequality, $\E\ip{u_s}{\tilde u_s}\leq\sigma_s\sigma_{s-h}$; therefore,
\begin{align*}
    (\sigma_s-\sigma_{s-h})^2
    &=
    \E\norm{u_s}^2+\E\norm{\tilde u_s}^2-2\sigma_s\sigma_{s-h}
    \\&\leq
    \E\norm{u_s}^2+\E\norm{\tilde u_s}^2-2\E\ip{u_s}{\tilde u_s}
    =
    \E\norm{u_s-\tilde u_s}^2,
\end{align*}
and summing over $s=r+h,\ldots,t$, and noting that $\sum_{s=r}^{t-h}(\sigma_{s+h}-\sigma_s)^2=\sum_{s=r+h}^t(\sigma_s-\sigma_{s-h})^2$, gives the first inequality. For the second inequality, we use the second property: as both sequences are updated with the sample $z_s$ at every step $s\geq r+h$, we have $x_{s+1}-\tilde x_{s+1}=x_s-\tilde x_s-\eta(u_s-\tilde u_s)$, and by \cref{lem:cocoercive} (in \cref{app:aux}) we know that $\ip{u_s-\tilde u_s}{x_s-\tilde x_s}\geq\frac1\beta\norm{u_s-\tilde u_s}^2$, thus, 
\begin{align*}
    \norm{x_{s+1}-\tilde x_{s+1}}^2
    &=
    \norm{x_s-\tilde x_s}^2-2\eta\ip{u_s-\tilde u_s}{x_s-\tilde x_s}+\eta^2\norm{u_s-\tilde u_s}^2
    \\&\leq
    \norm{x_s-\tilde x_s}^2-\frac{\eta(2-\beta\eta)}\beta\norm{u_s-\tilde u_s}^2.
\end{align*}
Summing over $s=r+h,\ldots,t$, the right-hand side telescopes, and we obtain
\begin{align*}
    \frac{\eta(2-\beta\eta)}\beta\sum_{s=r+h}^t\norm{u_s-\tilde u_s}^2
    \leq
    \norm{x_{r+h}-\tilde x_{r+h}}^2;
\end{align*}
taking expectations gives the second inequality.
\end{proof}

\subsection{The distance between the iterates at the starting point}
\label{sec:seq-start}

Next, we bound the distances $\E\norm{x_{r+h}-x_r}^2$ between the iterates at the starting point $r$ of the compared sequences, which appear in \cref{lem:comparison}. The following lemma bounds them by the risk at the starting point.
\begin{lemma}
\label{lem:displacement}
Under the conditions of \cref{thm:main}, for all integers $r,h\geq1$,
\begin{align*}
    \E\norm{x_{r+h}-x_r}^2
    \leq
    \frac{4\eta h}{2-\beta\eta}\Bigl(\E[F(x_r)-F(x^\star)]+\eta\sigma_\star^2\Bigr).
\end{align*}
\end{lemma}
The bound of \cref{lem:displacement} grows only linearly with the lag $h$, and depends only on the risk at the starting point $x_r$. Hence, combined with \cref{lem:comparison} and after the normalization by the lag $2^j$ in \cref{lem:endpoint}, it gives the same bound for all the lags.
Below we prove \cref{lem:displacement} by tracking the squared distance to a fixed reference point, as in a regret analysis against $x_r$.  
\begin{proof}[Proof of \cref{lem:displacement}]
\allowdisplaybreaks
We treat $x_r$ as a fixed reference point, and track the squared distance to it, as in a regret analysis against $x_r$. We first bound the expected squared distance to a fixed point after one step. Let $x,y\in\R^d$, and denote by $\E_z$ the expectation over $z\sim\cZ$. Then, by \cref{lem:cocoercive,lem:selfbound} and Young's inequality,
\begin{align*}
    &\E_z\norm{x-\eta\nabla f(x;z)-y}^2
    \\&\quad=
    \norm{x-y}^2-2\eta\E_z\ip{\nabla f(x;z)}{x-y}+\eta^2\E_z\norm{\nabla f(x;z)}^2
    \\&\quad\leq
    \norm{x-y}^2-2\eta\bigl(F(x)-F(y)\bigr)-\frac\eta\beta\E_z\norm{\nabla f(x;z)-\nabla f(y;z)}^2+\eta^2\E_z\norm{\nabla f(x;z)}^2
    \\&\quad\leq
    \norm{x-y}^2-2\eta\bigl(F(x)-F(y)\bigr)-\frac\eta\beta\E_z\norm{\nabla f(x;z)-\nabla f(y;z)}^2
    \\&\quad\qquad
    +\Bigl(\frac\eta\beta+\frac{\eta^2}2\Bigr)\E_z\norm{\nabla f(x;z)-\nabla f(x^\star;z)}^2+\frac{4\eta^2}{2-\beta\eta}\sigma_\star^2
    \\&\quad\leq
    \norm{x-y}^2+2\eta\bigl(F(y)-F(x^\star)\bigr)+\frac{4\eta^2}{2-\beta\eta}\sigma_\star^2
    \\&\quad\qquad
    +\frac\eta\beta\E_z\Bigl[\frac{\beta\eta}2\norm{\nabla f(x;z)-\nabla f(x^\star;z)}^2-\norm{\nabla f(x;z)-\nabla f(y;z)}^2\Bigr]
    \\&\quad\leq
    \norm{x-y}^2+2\eta\bigl(F(y)-F(x^\star)\bigr)+\frac{4\eta^2}{2-\beta\eta}\sigma_\star^2+\frac{\eta^2}{2-\beta\eta}\E_z\norm{\nabla f(y;z)-\nabla f(x^\star;z)}^2
    \\&\quad\leq
    \norm{x-y}^2+\frac{4\eta}{2-\beta\eta}\bigl(F(y)-F(x^\star)\bigr)+\frac{4\eta^2}{2-\beta\eta}\sigma_\star^2.
\end{align*}
where the fact that for all $u,w\in\R^d$,
\begin{align*}
    \frac{\beta\eta}2\norm{u}^2-\norm{u-w}^2
    =
    -\frac{2-\beta\eta}2\Bigl\lVert u-\frac2{2-\beta\eta}w\Bigr\rVert^2+\frac{\beta\eta}{2-\beta\eta}\norm{w}^2
    \leq
    \frac{\beta\eta}{2-\beta\eta}\norm{w}^2.
\end{align*}

Now, for every $s\in\{r,\ldots,r+h-1\}$, both $x_s$ and $x_r$ are determined by $z_1,\ldots,z_{s-1}$, and are therefore independent of $z_s$; thus, applying the last bound conditionally on $z_1,\ldots,z_{s-1}$ with $x=x_s$ and $y=x_r$, and taking expectations,
\begin{align*}
    \E\norm{x_{s+1}-x_r}^2
    &=
    \E\norm{x_s-\eta\nabla f_s(x_s)-x_r}^2
    \\&\leq
    \E\norm{x_s-x_r}^2+\frac{4\eta}{2-\beta\eta}\Bigl(\E[F(x_r)-F(x^\star)]+\eta\sigma_\star^2\Bigr).
\end{align*}
Summing over $s=r,\ldots,r+h-1$, the right-hand side telescopes, and the claim follows.
\end{proof}

\subsection{Bounding the stochastic gradient at the last iterate}
\label{sec:seq-gradient}

We now bound the last element $\sigma_t$ of the sequence, as established in the following lemma, 
\begin{lemma}
\label{lem:gradient}
Under the conditions of \cref{thm:main}, for every $t\geq1$,
\begin{align}\label{eq:o-gradient}
    \E\norm{\nabla f_t(x_t)}^2
    \leq
    128\biggl(\frac{\beta\norm{x_1-x^\star}^2}{\eta(2-\beta\eta)^3t}+\frac{8\sigma_\star^2}{(2-\beta\eta)^4}\biggr)\log^2(et).
\end{align}
\end{lemma}

The lemma follows by combining \cref{lem:endpoint,lem:comparison,lem:displacement} with a choice of the starting point $r$ of the compared sequences.

\begin{proof}[Proof of \cref{lem:gradient}]
Fix $t\geq1$, and denote $A\eqq\frac\beta t\sum_{s=1}^t\E[F(x_s)-F(x^\star)]$; throughout, we use that $0<\beta\eta<2$. We bound the two terms of \cref{lem:endpoint}, applied to $\sigma_1,\ldots,\sigma_t$. First, by \cref{lem:selfbound}, $\frac4t\sum_{s=1}^t\sigma_s^2\leq16A+8\sigma_\star^2$. Second, we choose $r\leq\lceil t/4\rceil$ deterministically as an index whose expected excess risk is the smallest among the first $\lceil t/4\rceil$ iterates, so that $\E[F(x_r)-F(x^\star)]\leq4A/\beta$, as the minimum is at most the average. For every lag $2^j$ in \cref{lem:endpoint}, we have $t\geq r+2^j$ by the choice of $J$, and thus, by \cref{lem:comparison,lem:displacement} with $h=2^j$,so
\begin{align*}
    \frac1{2^j}\sum_{s=r}^{t-2^j}(\sigma_{s+2^j}-\sigma_s)^2
    &\leq
    \frac\beta{\eta(2-\beta\eta)}\cdot\frac1{2^j}\E\norm{x_{r+2^j}-x_r}^2
    \leq
    \frac4{(2-\beta\eta)^2}\bigl(4A+2\sigma_\star^2\bigr).
\end{align*}
Plugging these bounds into \cref{lem:endpoint}, and using $1+\frac12\log^2t\leq\log^2(et)$,
\begin{align*}
    \E\norm{\nabla f_t(x_t)}^2
    \leq
    16A+8\sigma_\star^2+\frac{8\log^2t}{(2-\beta\eta)^2}\bigl(4A+2\sigma_\star^2\bigr)
    \leq
    \frac{64\log^2(et)}{(2-\beta\eta)^2}\bigl(A+2\sigma_\star^2\bigr).
\end{align*}
Finally, by the average-iterate analysis (\cref{lem:energy}), $A+2\sigma_\star^2\leq2\bigl(\frac{\beta\norm{x_1-x^\star}^2}{\eta(2-\beta\eta)t}+\frac{8\sigma_\star^2}{(2-\beta\eta)^2}\bigr)$, which gives the claim.
\end{proof}

\subsection{From the stochastic gradients to the risk of the last iterate}
\label{sec:seq-risk}

In the final step, we convert the bound of \cref{lem:gradient} on the stochastic gradients into a bound on the risk of the last iterate. We first relate the risk of the last iterate to the average risk of the iterates and to the stochastic gradients along the iterates.
\begin{lemma}
\label{lem:conversion}
Under the conditions of \cref{thm:main}, for every $T\geq1$,
\begin{align*}
    \E[F(x_T)-F(x^\star)]
    \leq
    \frac1T\sum_{t=1}^T\E[F(x_t)-F(x^\star)]+\frac\eta2\Biggl(\E\norm{\nabla f_T(x_T)}^2+\sum_{t=1}^{T-1}\frac{\E\norm{\nabla f_t(x_t)}^2}{T-t}\Biggr).
\end{align*}
\end{lemma}
The proof of \cref{lem:conversion} follows the suffix-averaging technique of \citet{shamir13sgd}: it combines the regret bound of SGD against past iterates with an inequality of \citet[Lemma~1]{orabona2020last}, which relates the last element of a sequence to its suffix averages (\cref{lem:suffix} in \cref{app:aux}).

\begin{proof}[Proof of \cref{lem:conversion}]
\allowdisplaybreaks
Fix $1\leq s\leq T$. For every $t\geq s$, by convexity of $f_t$ and the SGD update \cref{eq:sgd},
\begin{align*}
    f_t(x_t)-f_t(x_s)
    &\leq
    \ip{\nabla f_t(x_t)}{x_t-x_s}
    \\&=
    \frac1{2\eta}\Bigl(\norm{x_t-x_s}^2-\norm{x_t-\eta\nabla f_t(x_t)-x_s}^2\Bigr)+\frac\eta2\norm{\nabla f_t(x_t)}^2
    \\&=
    \frac1{2\eta}\bigl(\norm{x_t-x_s}^2-\norm{x_{t+1}-x_s}^2\bigr)+\frac\eta2\norm{\nabla f_t(x_t)}^2.
\end{align*}
Summing over $t=s,\ldots,T$, the right-hand side telescopes, and we have $$\sum_{t=s}^T(f_t(x_t)-f_t(x_s))\leq\frac\eta2\sum_{t=s}^T\norm{\nabla f_t(x_t)}^2.$$
Using $\E[f_t(x_t)-f_t(x_s)]=\E[f_t(x_t)-f_t(x^\star)]-\E[f_t(x_s)-f_t(x^\star)]=\E[F(x_t)-F(x_s)]$, we obtain, by taking expectations, that for every $1\leq s\leq T$,
\begin{align}\label{eq:regret}
    \sum_{t=s}^T\E[F(x_t)-F(x_s)]
    \leq
    \frac\eta2\sum_{t=s}^T\E\norm{\nabla f_t(x_t)}^2
\end{align}
Now, applying \cref{lem:suffix} with $v_t=\E[F(x_t)-F(x^\star)]$, which is nonnegative since $x^\star$ minimizes $F$ and the fact that $\sum_{k=a}^{b}\frac1{k(k+1)}=\frac1a-\frac1{b+1}\leq \frac{1}{a}$, 
\begin{align*}
    &\E[F(x_T)-F(x^\star)]
    \leq
    \frac1T\sum_{t=1}^T\E[F(x_t)-F(x^\star)]
    +\sum_{k=1}^{T-1}\frac1{k(k+1)}\sum_{t=T-k}^T\E[F(x_t)-F(x_{T-k})]
    \\&\leq
    \frac1T\sum_{t=1}^T\E[F(x_t)-F(x^\star)]+\frac\eta2\sum_{k=1}^{T-1}\frac1{k(k+1)}\sum_{t=T-k}^T\E\norm{\nabla f_t(x_t)}^2
    \tag{\cref{eq:regret} with $s=T-k$}
    \\&\leq\frac1T\sum_{t=1}^T\E[F(x_t)-F(x^\star)]
    +\frac\eta2\Biggl(\E\norm{\nabla f_T(x_T)}^2\sum_{k=1}^{T-1}\frac1{k(k+1)}+\sum_{t=1}^{T-1}\E\norm{\nabla f_t(x_t)}^2\sum_{k=T-t}^{T-1}\frac1{k(k+1)}\Biggr)
    \tag{exchanging the order of summation}
    \\&\leq
    \frac1T\sum_{t=1}^T\E[F(x_t)-F(x^\star)]
    +\frac\eta2\Biggl(\E\norm{\nabla f_T(x_T)}^2+\sum_{t=1}^{T-1}\frac{\E\norm{\nabla f_t(x_t)}^2}{T-t}\Biggr).\qedhere
\end{align*}
\end{proof}

Finally, we can prove \cref{thm:main}, by combining \cref{lem:conversion,lem:gradient} with the average-iterate bound of \cref{lem:energy}.

\begin{proof}[Proof of \cref{thm:main}]
\allowdisplaybreaks
Fix $T\geq1$. By \cref{lem:conversion,lem:energy},
\begin{align*}
    \E[F(x_T)-F(x^\star)]
    \leq
    \frac{2\norm{x_1-x^\star}^2}{\eta(2-\beta\eta)T}+\frac{8\eta\sigma_\star^2}{(2-\beta\eta)^2}
    +\frac\eta2\Biggl(\E\norm{\nabla f_T(x_T)}^2+\sum_{t=1}^{T-1}\frac{\E\norm{\nabla f_t(x_t)}^2}{T-t}\Biggr).
\end{align*}
To bound the last term, note that comparing the harmonic sum with an integral gives $\sum_{k=1}^{T-1}\frac1k\leq1+\log T=\log(eT)$ (where the sum is empty for $T=1$), and therefore
\begin{align*}
    \sum_{t=1}^{T-1}\frac1{t(T-t)}
    &=
    \frac1T\sum_{t=1}^{T-1}\Bigl(\frac1t+\frac1{T-t}\Bigr)
    \leq
    \frac{2\log(eT)}T,
   \qquad
    \sum_{t=1}^{T-1}\frac1{T-t}
    =
    \sum_{k=1}^{T-1}\frac1k
    \leq
    \log(eT).
\end{align*}
Thus, using \cref{lem:gradient} for every $t\leq T$,
\begin{align*}
    &\E\norm{\nabla f_T(x_T)}^2+\sum_{t=1}^{T-1}\frac{\E\norm{\nabla f_t(x_t)}^2}{T-t}
    \\&\qquad\leq
    128\log^2(eT)\Biggl(\frac{\beta\norm{x_1-x^\star}^2}{\eta(2-\beta\eta)^3}\Bigl(\frac1T+\sum_{t=1}^{T-1}\frac1{t(T-t)}\Bigr)+\frac{8\sigma_\star^2}{(2-\beta\eta)^4}\Bigl(1+\sum_{t=1}^{T-1}\frac1{T-t}\Bigr)\Biggr)
    \tag{\cref{lem:gradient}}
    \\&\qquad\leq
    128\log^3(eT)\biggl(\frac{3\beta\norm{x_1-x^\star}^2}{\eta(2-\beta\eta)^3T}+\frac{16\sigma_\star^2}{(2-\beta\eta)^4}\biggr).
\end{align*}
Combining everything gives,
\begin{align*}
    \E[F(x_T)-F(x^\star)]
    &\leq
    \frac{2\norm{x_1-x^\star}^2}{\eta(2-\beta\eta)T}+\frac{8\eta\sigma_\star^2}{(2-\beta\eta)^2}
    +64\eta\log^3(eT)\biggl(\frac{3\beta\norm{x_1-x^\star}^2}{\eta(2-\beta\eta)^3T}+\frac{16\sigma_\star^2}{(2-\beta\eta)^4}\biggr)
    \\&\leq
    \biggl(\frac{8\norm{x_1-x^\star}^2}{\eta(2-\beta\eta)^3T}+\frac{32\eta\sigma_\star^2}{(2-\beta\eta)^4}
    +\frac{384\norm{x_1-x^\star}^2}{\eta(2-\beta\eta)^3T}+\frac{1024\eta\sigma_\star^2}{(2-\beta\eta)^4}\biggr)\log^3(eT)
    \tag{$(2-\beta\eta)^2\leq4$, $\beta\eta\leq2$, $\log(eT)\geq1$}
    \\&\leq
    1056\biggl(\frac{\norm{x_1-x^\star}^2}{\eta(2-\beta\eta)^3T}+\frac{\eta\sigma_\star^2}{(2-\beta\eta)^4}\biggr)\log^3(eT). \qedhere
\end{align*}
\end{proof}

\section{Logarithm-free bounds for quadratic losses}
\label{sec:quadratic}

In this section, we prove our results for quadratic losses. 
Although our main result (\cref{thm:main}) captures the particular case of quadratic losses, we show that the logarithmic factors in \cref{thm:main} are not necessary for quadratic losses under interpolation. In particular, we establish the following theorem, 
(Here, $I$ denotes the identity matrix, and for symmetric matrices $A,B$, we write $A\preceq B$ if $B-A$ is positive semidefinite.)
\begin{theorem}
\label{thm:quadratic}
Under the conditions of \cref{thm:main}, assume in addition that for every $z\in Z$, there is a symmetric matrix $H_z$ with $0\preceq H_z\preceq\beta I$ such that $f(x;z)-f(x^\star;z)=\frac12(x-x^\star)^\top H_z(x-x^\star)$ for all $x\in\R^d$. Then, for every $t\geq1$,
\begin{align*}
    \E[F(x_t)-F(x^\star)]
    \leq
    \frac{(10-\beta\eta)\norm{x_1-x^\star}^2}{2\eta(2-\beta\eta)^3t}.
\end{align*}
In particular, for $\eta=\frac1\beta$, it holds that $\E[F(x_t)-F(x^\star)]\leq9\beta\norm{x_1-x^\star}^2/(2t)$.
\end{theorem}

The proof of \cref{thm:quadratic} is based on the following lemma, which bounds the one-step decrease of the squared distance to $x^\star$.
\begin{lemma}
\label{lem:quad-decrease}
Under the conditions of \cref{thm:quadratic}, for every $t\geq1$,
\begin{align*}
    \E\norm{x_t-x^\star}^2-\E\norm{x_{t+1}-x^\star}^2
    \leq
    \frac1t\biggl(\sum_{r=1}^t\E\norm{x_{r+1}-x_r}^2+\frac{\norm{x_1-x^\star}^2}{2-\beta\eta}\biggr).
\end{align*}
\end{lemma}
As the proof of \cref{lem:comparison} in the general case, the proof of \cref{lem:quad-decrease} uses the virtual iterates defined in \cref{sec:seq-comparison}. However, here, the proofs compares SGD with the virtual iterates associated with all the starting points $r\in\{1,\ldots,t\}$ simultaneously, with the lag $h=1$. In contrast to the general case, SGD and the virtual iterates are updated by the same linear maps; thus, the effects of the $t$ deleted updates combine linearly, and their total effect admits a bound that does not depend on $t$. 

\Cref{lem:quad-decrease} is closely related to the operator inequality of \citet[Theorem~3.1]{derezinski2026}: both start from the same deletion identity, and differ in the bound on the combined effect of the deletions, which \citet{derezinski2026} obtain by complex analysis, and which we obtain by standard convex optimization techniques that apply to every stepsize $0<\eta<2/\beta$.

As a consequence of \cref{lem:quad-decrease}, we obtain the following analogues of \cref{lem:endpoint,lem:conversion}, which do not involve logarithmic factors and their proofs are deferred to \cref{app:quadratic}.
\begin{lemma}
\label{lem:quad-reduction}
Under the conditions of \cref{thm:quadratic}, for every $t\geq1$,
\begin{align*}
    \sigma_t^2
    \leq
    \frac{\beta}{\eta(2-\beta\eta)t}\biggl(\sum_{r=1}^t\E\norm{x_{r+1}-x_r}^2+\frac{\norm{x_1-x^\star}^2}{2-\beta\eta}\biggr).
\end{align*}
\end{lemma}
\begin{lemma}
\label{lem:quad-conversion}
Under the conditions of \cref{thm:quadratic}, for every $t\geq1$,
\begin{align*}
    \E[F(x_t)-F(x^\star)]
    \leq
    \frac\eta4\E\norm{\nabla f_t(x_t)}^2+\frac1{4\eta t}\biggl(\sum_{r=1}^t\E\norm{x_{r+1}-x_r}^2+\frac{\norm{x_1-x^\star}^2}{2-\beta\eta}\biggr).
\end{align*}
\end{lemma}

Finally, the proof of \cref{thm:quadratic} is obtained by combining \cref{lem:quad-reduction,lem:quad-conversion} with the average-iterate analysis (\cref{lem:energy} in \cref{app:aux}).

\begin{proof}[Proof of \cref{thm:quadratic}]
Fix $t\geq1$, and denote by $B\eqq\sum_{r=1}^t\E\norm{x_{r+1}-x_r}^2+\frac{\norm{x_1-x^\star}^2}{2-\beta\eta}$.  Since $\sigma_\star=0$, by \cref{lem:displacement} with $h=1$ and \cref{lem:energy}, 
\begin{align*}
    B
    \leq
    \frac{8\norm{x_1-x^\star}^2}{(2-\beta\eta)^2}+\frac{\norm{x_1-x^\star}^2}{2-\beta\eta}
    =
    \frac{(10-\beta\eta)\norm{x_1-x^\star}^2}{(2-\beta\eta)^2}.
\end{align*}
By \cref{lem:quad-reduction}, we get that $\E\norm{\nabla f_t(x_t)}^2\leq\frac{\beta B}{\eta(2-\beta\eta)t}$, and plugging this into \cref{lem:quad-conversion} gives
\begin{align*}
    \E[F(x_t)-F(x^\star)]
    \leq
    \frac{\beta B}{4(2-\beta\eta)t}+\frac{B}{4\eta t}
    =
    \frac{B}{2\eta(2-\beta\eta)t}
    \leq
    \frac{(10-\beta\eta)\norm{x_1-x^\star}^2}{2\eta(2-\beta\eta)^3t}.
    \qedhere
\end{align*}
\end{proof}

\section*{Acknowledgments}

This project has received funding from the European Research Council (ERC) under the European
Union's Horizon 2020 research and innovation program (grant agreement No.\ 101078075). Views and
opinions expressed are however those of the author(s) only and do not necessarily reflect those of the
European Union or the European Research Council. Neither the European Union nor the granting
authority can be held responsible for them. This work received additional support from the Israel
Science Foundation (ISF; grant numbers 2549/19 and 3174/23), from the Council for Higher Education in Israel under a Moonshot Project, and a fellowship from the Tel Aviv University Center for AI and Data Science (TAD).

\bibliographystyle{plainnat}
\bibliography{matan_references}

\begin{thebibliography}{58}
\providecommand{\natexlab}[1]{#1}
\providecommand{\url}[1]{\texttt{#1}}
\expandafter\ifx\csname urlstyle\endcsname\relax
  \providecommand{\doi}[1]{doi: #1}\else
  \providecommand{\doi}{doi: \begingroup \urlstyle{rm}\Url}\fi

\bibitem[Attia et~al.(2026)Attia, Schliserman, Sherman, and Koren]{attia2025fast}
Amit Attia, Matan Schliserman, Uri Sherman, and Tomer Koren.
\newblock Fast last-iterate convergence of sgd in the smooth interpolation regime.
\newblock \emph{Advances in Neural Information Processing Systems}, 38:\penalty0 104951--104987, 2026.

\bibitem[Bach and Moulines(2013)]{bach2013non}
Francis Bach and Eric Moulines.
\newblock Non-strongly-convex smooth stochastic approximation with convergence rate ${O}(1/n)$.
\newblock \emph{Advances in neural information processing systems}, 26, 2013.

\bibitem[Berthier et~al.(2020)Berthier, Bach, and Gaillard]{berthier2020tight}
Rapha{\"e}l Berthier, Francis Bach, and Pierre Gaillard.
\newblock Tight nonparametric convergence rates for stochastic gradient descent under the noiseless linear model.
\newblock \emph{Advances in Neural Information Processing Systems}, 33:\penalty0 2576--2586, 2020.

\bibitem[Boyd et~al.(2003)Boyd, Dattorro, et~al.]{boyd2003alternating}
Stephen Boyd, Jon Dattorro, et~al.
\newblock Alternating projections.
\newblock \emph{EE392o, Stanford University}, 2003.

\bibitem[Cai and Diakonikolas(2025)]{cai2025lastIterate}
Xufeng Cai and Jelena Diakonikolas.
\newblock Last iterate convergence of incremental methods and applications in continual learning.
\newblock In \emph{The Thirteenth International Conference on Learning Representations}, 2025.

\bibitem[Cortild et~al.(2025)Cortild, Ketels, Peypouquet, and Garrigos]{cortild2025new}
Daniel Cortild, Lucas Ketels, Juan Peypouquet, and Guillaume Garrigos.
\newblock New tight bounds for {SGD} without variance assumption: A computer-aided {L}yapunov analysis.
\newblock \emph{arXiv preprint arXiv:2505.17965}, 2025.

\bibitem[Derezi{\'n}ski and Dong(2026{\natexlab{a}})]{derezinski2026}
Micha{\l} Derezi{\'n}ski and Xiaoyu Dong.
\newblock How {AI} settled the complexity of the oldest {SGD} algorithm.
\newblock \emph{arXiv preprint arXiv:2606.29593}, 2026{\natexlab{a}}.

\bibitem[Derezi{\'n}ski and Dong(2026{\natexlab{b}})]{derezinski2026colt}
Micha{\l} Derezi{\'n}ski and Xiaoyu Dong.
\newblock Last-iterate convergence of randomized {K}aczmarz and {SGD} with greedy step size.
\newblock In \emph{Conference on Learning Theory (COLT)}, 2026{\natexlab{b}}.
\newblock arXiv:2604.09909.

\bibitem[Doan et~al.(2021)Doan, Abbana~Bennani, Mazoure, Rabusseau, and Alquier]{doan2021NTKoverlap}
Thang Doan, Mehdi Abbana~Bennani, Bogdan Mazoure, Guillaume Rabusseau, and Pierre Alquier.
\newblock A theoretical analysis of catastrophic forgetting through the {NTK} overlap matrix.
\newblock In \emph{Proceedings of The 24th International Conference on Artificial Intelligence and Statistics}, pages 1072--1080, 2021.

\bibitem[Elfving(1980)]{elfving1980block}
Tommy Elfving.
\newblock Block-iterative methods for consistent and inconsistent linear equations.
\newblock \emph{Numerische Mathematik}, 35\penalty0 (1):\penalty0 1--12, 1980.

\bibitem[Evron et~al.(2022)Evron, Moroshko, Ward, Srebro, and Soudry]{evron2022catastrophic}
Itay Evron, Edward Moroshko, Rachel Ward, Nathan Srebro, and Daniel Soudry.
\newblock How catastrophic can catastrophic forgetting be in linear regression?
\newblock In \emph{Conference on Learning Theory (COLT)}, pages 4028--4079. PMLR, 2022.

\bibitem[Evron et~al.(2023)Evron, Moroshko, Buzaglo, Khriesh, Marjieh, Srebro, and Soudry]{evron23continualClassification}
Itay Evron, Edward Moroshko, Gon Buzaglo, Maroun Khriesh, Badea Marjieh, Nathan Srebro, and Daniel Soudry.
\newblock Continual learning in linear classification on separable data.
\newblock In \emph{International Conference on Machine Learning}, pages 9440--9484. PMLR, 2023.

\bibitem[Evron et~al.(2026)Evron, Levinstein, Schliserman, Sherman, Koren, Soudry, and Srebro]{evron2025better}
Itay Evron, Ran Levinstein, Matan Schliserman, Uri Sherman, Tomer Koren, Daniel Soudry, and Nathan Srebro.
\newblock From continual learning to sgd and back: Better rates for continual linear models.
\newblock In Matus Telgarsky and Jonathan Ullman, editors, \emph{Proceedings of The 37th International Conference on Algorithmic Learning Theory}, volume 313 of \emph{Proceedings of Machine Learning Research}, pages 1--50. PMLR, 23--26 Feb 2026.
\newblock URL \url{https://proceedings.mlr.press/v313/evron26a.html}.

\bibitem[Garrigos et~al.(2025)Garrigos, Cortild, Ketels, and Peypouquet]{garrigos2025last}
Guillaume Garrigos, Daniel Cortild, Lucas Ketels, and Juan Peypouquet.
\newblock Last-iterate complexity of {SGD} for convex and smooth stochastic problems.
\newblock \emph{arXiv preprint arXiv:2507.14122}, 2025.

\bibitem[Ge et~al.(2019)Ge, Kakade, Kidambi, and Netrapalli]{ge2019stepDecay}
Rong Ge, Sham~M Kakade, Rahul Kidambi, and Praneeth Netrapalli.
\newblock The step decay schedule: A near optimal, geometrically decaying learning rate procedure for least squares.
\newblock \emph{Advances in Neural Information Processing Systems}, 32, 2019.

\bibitem[Goldfarb and Hand(2023)]{goldfarb2023analysis}
Daniel Goldfarb and Paul Hand.
\newblock Analysis of catastrophic forgetting for random orthogonal transformation tasks in the overparameterized regime.
\newblock In \emph{International Conference on Artificial Intelligence and Statistics}, pages 2975--2993. PMLR, 2023.

\bibitem[Goldfarb et~al.(2024)Goldfarb, Evron, Weinberger, Soudry, and Hand]{goldfarb2024theJointEffect}
Daniel Goldfarb, Itay Evron, Nir Weinberger, Daniel Soudry, and Paul Hand.
\newblock The joint effect of task similarity and overparameterization on catastrophic forgetting - an analytical model.
\newblock In \emph{The Twelfth International Conference on Learning Representations}, 2024.

\bibitem[Gower and Richt{\'a}rik(2015)]{gower2015randomized}
Robert~M Gower and Peter Richt{\'a}rik.
\newblock Randomized iterative methods for linear systems.
\newblock \emph{SIAM Journal on Matrix Analysis and Applications}, 36\penalty0 (4):\penalty0 1660--1690, 2015.

\bibitem[Gubin et~al.(1967)Gubin, Polyak, and Raik]{gubin1967methodProjections}
LG~Gubin, Boris~T Polyak, and EV~Raik.
\newblock The method of projections for finding the common point of convex sets.
\newblock \emph{USSR Computational Mathematics and Mathematical Physics}, 7\penalty0 (6):\penalty0 1--24, 1967.

\bibitem[Hardt et~al.(2016)Hardt, Recht, and Singer]{hardt2016train}
Moritz Hardt, Ben Recht, and Yoram Singer.
\newblock Train faster, generalize better: Stability of stochastic gradient descent.
\newblock In \emph{International Conference on Machine Learning}, pages 1225--1234. PMLR, 2016.

\bibitem[Harvey et~al.(2019)Harvey, Liaw, Plan, and Randhawa]{harvey2019tight}
Nicholas~JA Harvey, Christopher Liaw, Yaniv Plan, and Sikander Randhawa.
\newblock Tight analyses for non-smooth stochastic gradient descent.
\newblock In \emph{Conference on Learning Theory}, pages 1579--1613. PMLR, 2019.

\bibitem[Hemati et~al.(2024)Hemati, Pellegrini, Duan, Zhao, Xia, Masana, Tscheschner, Veas, Zheng, Zhao, et~al.]{hemati2024continual}
Hamed Hemati, Lorenzo Pellegrini, Xiaotian Duan, Zixuan Zhao, Fangfang Xia, Marc Masana, Benedikt Tscheschner, Eduardo Veas, Yuxiang Zheng, Shiji Zhao, et~al.
\newblock Continual learning in the presence of repetition.
\newblock In \emph{CVPR Workshop on Continual Learning in Computer Vision}, 2024.

\bibitem[Hiratani(2024)]{hiratani2024disentangling}
Naoki Hiratani.
\newblock Disentangling and mitigating the impact of task similarity for continual learning.
\newblock In \emph{The Thirty-eighth Annual Conference on Neural Information Processing Systems}, 2024.

\bibitem[Jain et~al.(2019)Jain, Nagaraj, and Netrapalli]{jain2019making}
Prateek Jain, Dheeraj Nagaraj, and Praneeth Netrapalli.
\newblock Making the last iterate of {SGD} information theoretically optimal.
\newblock In \emph{Conference on Learning Theory}, pages 1752--1755. PMLR, 2019.

\bibitem[Jung et~al.(2025)Jung, Cho, and Yun]{jung2025convergence}
Hyunji Jung, Hanseul Cho, and Chulhee Yun.
\newblock Convergence and implicit bias of gradient descent on continual linear classification.
\newblock In \emph{The Thirteenth International Conference on Learning Representations}, 2025.

\bibitem[Kaczmarz(1937)]{karczmarz1937angenaherte}
S~Kaczmarz.
\newblock Angenaherte auflosung von systemen linearer glei-chungen.
\newblock \emph{Bull. Int. Acad. Pol. Sic. Let., Cl. Sci. Math. Nat.}, pages 355--357, 1937.

\bibitem[Kornowski and Shamir(2026)]{kornowski2026}
Guy Kornowski and Ohad Shamir.
\newblock Gradient descent's last iterate is often (slightly) suboptimal.
\newblock \emph{arXiv preprint arXiv:2604.13870}, 2026.

\bibitem[Lan(2012)]{lan2012optimal}
Guanghui Lan.
\newblock An optimal method for stochastic composite optimization.
\newblock \emph{Mathematical Programming}, 133\penalty0 (1):\penalty0 365--397, 2012.

\bibitem[Lei and Ying(2020)]{lei20stab}
Yunwen Lei and Yiming Ying.
\newblock Fine-grained analysis of stability and generalization for stochastic gradient descent.
\newblock In Hal~Daumé III and Aarti Singh, editors, \emph{Proceedings of the 37th International Conference on Machine Learning}, volume 119 of \emph{Proceedings of Machine Learning Research}, pages 5809--5819. PMLR, 13--18 Jul 2020.

\bibitem[Lesort et~al.(2023)Lesort, Ostapenko, Rodr{\'\i}guez, Misra, Arefin, Charlin, and Rish]{lesort2022scaling}
Timoth{\'e}e Lesort, Oleksiy Ostapenko, Pau Rodr{\'\i}guez, Diganta Misra, Md~Rifat Arefin, Laurent Charlin, and Irina Rish.
\newblock Challenging common assumptions about catastrophic forgetting and knowledge accumulation.
\newblock In \emph{Conference on Lifelong Learning Agents}, pages 43--65. PMLR, 2023.

\bibitem[Levinstein et~al.(2026)Levinstein, Attia, Schliserman, Sherman, Soudry, Koren, and Evron]{levinstein2025optimal}
Ran Levinstein, Amit Attia, Matan Schliserman, Uri Sherman, Daniel Soudry, Tomer Koren, and Itay Evron.
\newblock Optimal rates in continual linear regression via increasing regularization.
\newblock \emph{Advances in Neural Information Processing Systems}, 38:\penalty0 60456--60495, 2026.

\bibitem[Li et~al.(2023)Li, Wu, and Braverman]{li2023fixed}
Haoran Li, Jingfeng Wu, and Vladimir Braverman.
\newblock Fixed design analysis of regularization-based continual learning.
\newblock In \emph{Conference on lifelong learning agents}, pages 513--533. PMLR, 2023.

\bibitem[Lin et~al.(2023)Lin, Ju, Liang, and Shroff]{lin2023theory}
Sen Lin, Peizhong Ju, Yingbin Liang, and Ness Shroff.
\newblock Theory on forgetting and generalization of continual learning.
\newblock In \emph{International Conference on Machine Learning}, pages 21078--21100. PMLR, 2023.

\bibitem[Liu et~al.(2023)Liu, Drusvyatskiy, Belkin, Davis, and Ma]{liu2023aiming}
Chaoyue Liu, Dmitriy Drusvyatskiy, Misha Belkin, Damek Davis, and Yian Ma.
\newblock Aiming towards the minimizers: fast convergence of {SGD} for overparametrized problems.
\newblock \emph{Advances in Neural Information Processing Systems}, 36:\penalty0 60748--60767, 2023.

\bibitem[Liu(2026)]{liu2026random}
Zijian Liu.
\newblock Random reshuffling dominates stochastic gradient descent.
\newblock In Steve Hanneke and Tor Lattimore, editors, \emph{Proceedings of Thirty Ninth Conference on Learning Theory}, volume 336 of \emph{Proceedings of Machine Learning Research}, pages 4859--4882. PMLR, 29 Jun--03 Jul 2026.
\newblock URL \url{https://proceedings.mlr.press/v336/liu26d.html}.

\bibitem[Liu and Zhou(2024)]{liu2024revisiting}
Zijian Liu and Zhengyuan Zhou.
\newblock Revisiting the last-iterate convergence of stochastic gradient methods.
\newblock In \emph{The Twelfth International Conference on Learning Representations}, 2024.

\bibitem[Ma et~al.(2018)Ma, Bassily, and Belkin]{ma2018power}
Siyuan Ma, Raef Bassily, and Mikhail Belkin.
\newblock The power of interpolation: Understanding the effectiveness of {SGD} in modern over-parametrized learning.
\newblock In \emph{International Conference on Machine Learning}, pages 3325--3334. PMLR, 2018.

\bibitem[Mermillod et~al.(2013)Mermillod, Bugaiska, and Bonin]{mermillod2013stability}
Martial Mermillod, Aur{\'e}lia Bugaiska, and Patrick Bonin.
\newblock The stability-plasticity dilemma: Investigating the continuum from catastrophic forgetting to age-limited learning effects, 2013.

\bibitem[Needell and Tropp(2014)]{needell2014paved}
Deanna Needell and Joel~A Tropp.
\newblock Paved with good intentions: analysis of a randomized block {K}aczmarz method.
\newblock \emph{Linear Algebra and its Applications}, 441:\penalty0 199--221, 2014.

\bibitem[Needell et~al.(2014)Needell, Ward, and Srebro]{needell2014stochastic}
Deanna Needell, Rachel Ward, and Nati Srebro.
\newblock Stochastic gradient descent, weighted sampling, and the randomized {K}aczmarz algorithm.
\newblock \emph{Advances in Neural Information Processing Systems}, 27:\penalty0 1017--1025, 2014.

\bibitem[Nesterov(1998)]{nesterov1998introductory}
Yurii Nesterov.
\newblock Introductory lectures on convex programming volume {I}: Basic course.
\newblock \emph{Lecture notes}, 3\penalty0 (4):\penalty0 5, 1998.

\bibitem[Orabona(2020)]{orabona2020last}
Francesco Orabona.
\newblock Last iterate of {SGD} converges (even in unbounded domains).
\newblock Blog post, 2020.
\newblock URL \url{https://parameterfree.com/2020/08/07/last-iterate-of-sgd-converges-even-in-unbounded-domains}.

\bibitem[Oswald and Zhou(2015)]{oswald2015convergence}
Peter Oswald and Weiqi Zhou.
\newblock Convergence analysis for {K}aczmarz-type methods in a {H}ilbert space framework.
\newblock \emph{Linear Algebra and its Applications}, 478:\penalty0 131--161, 2015.

\bibitem[Peng et~al.(2023)Peng, Giampouras, and Vidal]{peng2023ideal}
Liangzu Peng, Paris Giampouras, and Ren{\'e} Vidal.
\newblock The ideal continual learner: An agent that never forgets.
\newblock In \emph{International Conference on Machine Learning}, 2023.

\bibitem[Rakhlin et~al.(2012)Rakhlin, Shamir, and Sridharan]{ICML2012Rakhlin_261}
Alexander Rakhlin, Ohad Shamir, and Karthik Sridharan.
\newblock Making gradient descent optimal for strongly convex stochastic optimization.
\newblock In \emph{International Conference on Machine Learning}, pages 449--456. PMLR, 2012.

\bibitem[Schliserman and Koren(2022)]{schliserman2022stability}
Matan Schliserman and Tomer Koren.
\newblock Stability vs implicit bias of gradient methods on separable data and beyond.
\newblock In \emph{Conference on Learning Theory}, pages 3380--3394. PMLR, 2022.

\bibitem[Schliserman et~al.(2026)Schliserman, Buzaglo, Evron, and Soudry]{schliserman2026convergence}
Matan Schliserman, Gon Buzaglo, Itay Evron, and Daniel Soudry.
\newblock Convergence of continual learning in homogeneous deep networks.
\newblock In Steve Hanneke and Tor Lattimore, editors, \emph{Proceedings of Thirty Ninth Conference on Learning Theory}, volume 336 of \emph{Proceedings of Machine Learning Research}, pages 5743--5784. PMLR, 29 Jun--03 Jul 2026.
\newblock URL \url{https://proceedings.mlr.press/v336/schliserman26a.html}.

\bibitem[Shamir and Zhang(2013)]{shamir13sgd}
Ohad Shamir and Tong Zhang.
\newblock Stochastic gradient descent for non-smooth optimization: Convergence results and optimal averaging schemes.
\newblock In Sanjoy Dasgupta and David McAllester, editors, \emph{Proceedings of the 30th International Conference on Machine Learning}, volume~28 of \emph{Proceedings of Machine Learning Research}, pages 71--79. PMLR, 2013.

\bibitem[Sherman et~al.(2021)Sherman, Koren, and Mansour]{sherman2021optimal}
Uri Sherman, Tomer Koren, and Yishay Mansour.
\newblock Optimal rates for random order online optimization.
\newblock \emph{Advances in Neural Information Processing Systems}, 34:\penalty0 2097--2108, 2021.

\bibitem[Srebro et~al.(2010)Srebro, Sridharan, and Tewari]{srebro2010smoothness}
Nathan Srebro, Karthik Sridharan, and Ambuj Tewari.
\newblock Smoothness, low noise and fast rates.
\newblock \emph{Advances in Neural Information Processing Systems}, 23, 2010.

\bibitem[Strohmer and Vershynin(2009)]{strohmer2009randomized}
Thomas Strohmer and Roman Vershynin.
\newblock A randomized {K}aczmarz algorithm with exponential convergence.
\newblock \emph{Journal of Fourier Analysis and Applications}, 15\penalty0 (2):\penalty0 262--278, 2009.

\bibitem[Swartworth et~al.(2023)Swartworth, Needell, Ward, Kong, and Jeong]{swartworth2023nearly}
William Swartworth, Deanna Needell, Rachel Ward, Mark Kong, and Halyun Jeong.
\newblock Nearly optimal bounds for cyclic forgetting.
\newblock \emph{Advances in neural information processing systems}, 36:\penalty0 68197--68206, 2023.

\bibitem[Varre et~al.(2021)Varre, Pillaud-Vivien, and Flammarion]{varre2021last}
Aditya~Vardhan Varre, Loucas Pillaud-Vivien, and Nicolas Flammarion.
\newblock Last iterate convergence of {SGD} for least-squares in the interpolation regime.
\newblock In A.~Beygelzimer, Y.~Dauphin, P.~Liang, and J.~Wortman Vaughan, editors, \emph{Advances in Neural Information Processing Systems}, 2021.

\bibitem[Vaswani et~al.(2019)Vaswani, Bach, and Schmidt]{vaswani2019fast}
Sharan Vaswani, Francis Bach, and Mark Schmidt.
\newblock Fast and faster convergence of {SGD} for over-parameterized models and an accelerated perceptron.
\newblock In \emph{The 22nd international conference on artificial intelligence and statistics}, pages 1195--1204. PMLR, 2019.

\bibitem[Wu et~al.(2022)Wu, Zou, Braverman, Gu, and Kakade]{wu2022last}
Jingfeng Wu, Difan Zou, Vladimir Braverman, Quanquan Gu, and Sham Kakade.
\newblock Last iterate risk bounds of {SGD} with decaying stepsize for overparameterized linear regression.
\newblock In \emph{International Conference on Machine Learning}, pages 24280--24314. PMLR, 2022.

\bibitem[Yu and Diakonikolas(2026)]{yu2026tight}
Runbo Yu and Jelena Diakonikolas.
\newblock Tight convergence bounds for the classical kaczmarz method.
\newblock \emph{arXiv preprint arXiv:2609.32034}, 2026.

\bibitem[Zamani and Glineur(2025)]{zamani2023exact}
Moslem Zamani and Fran{\c{c}}ois Glineur.
\newblock Exact convergence rate of the last iterate in subgradient methods.
\newblock \emph{SIAM Journal on Optimization}, 35\penalty0 (3):\penalty0 2182--2201, 2025.

\bibitem[Zhang et~al.(2023)Zhang, Teng, and Zhang]{zhang2023lower}
Peiyuan Zhang, Jiaye Teng, and Jingzhao Zhang.
\newblock Lower generalization bounds for gd and sgd in smooth stochastic convex optimization.
\newblock \emph{arXiv preprint arXiv:2303.10758}, 2023.

\end{thebibliography}

\clearpage
\appendix
\crefalias{section}{appendix}
\crefalias{subsection}{appendix}

\section{Applications}
\label{app:applications}

In this section, we apply our results to random projections onto convex sets, the randomized Kaczmarz method, and forgetting in continual learning. Each setting can be expressed as SGD in the interpolation regime by choosing an appropriate loss.

\subsection{Random projections and randomized Kaczmarz}
\label{app:projections}

The random projection method finds a point in the intersection of convex sets~\citep{gubin1967methodProjections,boyd2003alternating}. For a nonempty closed convex set $C\subseteq\R^d$, let $P_Cx$ denote the Euclidean projection of $x$ onto $C$, and write $\operatorname{dist}(x,C)\eqq\norm{x-P_Cx}$. Consider a family of such sets $C_z$, $z\in Z$, with a common point $x^\star\in C_z$ for almost every $z\sim\cZ$. Starting from $x_1\in\R^d$, the method draws $z_1,z_2,\ldots$ i.i.d.\ from $\cZ$ and updates
\begin{equation}
\label{eq:projection-update}
    x_{t+1}\gets P_{C_{z_t}}x_t,
\end{equation}
and returns the last iterate $x_T$.

To express this method as SGD, consider the losses $f(x;z)=\frac12\operatorname{dist}(x,C_z)^2$. Each loss is convex and differentiable, with gradient $\nabla f(x;z)=x-P_{C_z}x$. This gradient is $1$-Lipschitz because $I-P_{C_z}$ is firmly nonexpansive. Thus, \cref{eq:projection-update} is SGD with $\beta=1$ and the greedy stepsize $\eta=1$. The common solution $x^\star$ minimizes the losses almost surely, so $\sigma_\star=0$.

The identity $\norm{\nabla f(x;z)}^2=\operatorname{dist}(x,C_z)^2$ allows us to apply \cref{lem:gradient} directly. 
In particylar, we get the following corollary, which bounds the expected squared distance of the last iterate from a fresh random set.

\begin{corollary}
\label{cor:projections}
In the setting above, the random projection method (\cref{eq:projection-update}) satisfies, for every $T\geq1$,
\begin{align*}
    \E\bigl[\operatorname{dist}(x_T,C_{z_T})^2\bigr]
    \leq
    \frac{128\norm{x_1-x^\star}^2\log^2(eT)}{T}.
\end{align*}
\end{corollary}

\begin{proof}
The left-hand side is $\E\norm{\nabla f_T(x_T)}^2$. Apply \cref{lem:gradient} with $\beta=\eta=1$ and $\sigma_\star=0$.
\end{proof}

The randomized Kaczmarz method solves a consistent linear system $Ax=b$ by projecting onto the solution set of one equation at a time~\citep{karczmarz1937angenaherte,strohmer2009randomized}. Let $a_1^\top,\ldots,a_n^\top$ be the nonzero rows of $A$. At each iteration $t$, the method independently draws a row index $i_t$ with probabilities $\pi_1,\ldots,\pi_n>0$ and updates
\begin{equation}
\label{eq:kaczmarz-update}
    x_{t+1}\gets x_t-\frac{a_{i_t}^\top x_t-b_{i_t}}{\norm{a_{i_t}}^2}\,a_{i_t}.
\end{equation}
This is the special case of \cref{eq:projection-update} in which the sets are the hyperplanes $\{x\colon a_i^\top x=b_i\}$. We measure the performance of the last iterate by its weighted residual, which admits a bound independent of the condition number of $A$.

Previous work (e.g. \cite{needell2014stochastic,evron2025better}) has shown that the randomized Kaczmarz method can be viewed as SGD with stepsize $\eta=1$ on the losses $\frac{(a_i^\top x-b_i)^2}{2\norm{a_i}^2}$. For any solution $x^\star$ of $Ax=b$, these losses equal $\frac12(x-x^\star)^\top P_i(x-x^\star)$, where $P_i\eqq a_ia_i^\top/\norm{a_i}^2$ is the orthogonal projection onto the span of $a_i$. 
Since the functions are quadratic with $0\preceq P_i\preceq I$, we can apply \cref{thm:quadratic} and obtain the following corollary.
\begin{corollary}[randomized Kaczmarz]
\label{cor:kaczmarz}
In the setting above, the randomized Kaczmarz method (\cref{eq:kaczmarz-update}) satisfies, for every $T\geq1$,
\begin{align*}
    \E\biggl[\sum_{i=1}^n\pi_i\frac{(a_i^\top x_T-b_i)^2}{\norm{a_i}^2}\biggr]
    \leq
    \frac{9\norm{x_1-x^\star}^2}{T}.
\end{align*}
In particular, if $\pi_i=\norm{a_i}^2/\norm A_F^2$ for all $i$, then $\E\norm{Ax_T-b}^2\leq9\norm A_F^2\norm{x_1-x^\star}^2/T$.
\end{corollary}

\begin{proof}
The weighted residual is $2(F(x_T)-F(x^\star))$, so the first claim follows from \cref{thm:quadratic} with $\beta=\eta=1$. When $\pi_i=\norm{a_i}^2/\norm A_F^2$, this residual equals $\norm{Ax_T-b}^2/\norm A_F^2$, giving the second claim.
\end{proof}

The bound in \cref{cor:kaczmarz} improves upon the rates of \citet{evron2025better} and \citet{attia2025fast}, and matches the rate of \citet{derezinski2026}. The same argument applies to the block Kaczmarz method~\citep{elfving1980block}: projecting onto the solution set of a block of equations again corresponds to a quadratic loss whose Hessian is an orthogonal projection.

\subsection{Forgetting in continual learning}
\label{app:continual}

Continual learning of linear models has received considerable attention in recent years~\citep[e.g.,][]{doan2021NTKoverlap,evron2022catastrophic,lin2023theory,evron2025better,levinstein2025optimal}. We consider tasks $z_1,z_2,\ldots$ drawn i.i.d.\ from $\cZ$, with each task $z$ associated with a nonempty closed convex solution set $C_z$. Assume the tasks are jointly realizable: there is a point $x^\star\in C_z$ for almost every $z\sim\cZ$. Starting from $x_1\in\R^d$, the learner fits each task by projecting its current model onto the task's solution set:
\begin{equation}
\label{eq:continual-update}
    x_{t+1}\gets\argmin_{x\in C_{z_t}}\norm{x-x_t}^2=P_{C_{z_t}}x_t.
\end{equation}
This framework includes continual linear regression trained to convergence by gradient methods, where the sets $C_z$ are affine~\citep{evron2022catastrophic}, and weakly regularized continual linear classification on jointly separable data, where they are polyhedral~\citep{evron23continualClassification}. After $k$ tasks, the learner returns $x_{k+1}$. Following \citet{evron2022catastrophic}, we measure \emph{forgetting} by the average loss on the observed tasks ,defeined as, $
    \frac1k\sum_{t=1}^k\frac12\operatorname{dist}(x_{k+1},C_{z_t})^2.
$
Each task has zero loss immediately after it is learned, so this quantity also measures the average increase in loss since learning each task.

As in previous work~\citep{evron2022catastrophic,evron23continualClassification,evron2025better}, we view \cref{eq:continual-update} as SGD with stepsize $\eta=1$ on the losses $f(x;z)=\frac12\operatorname{dist}(x,C_z)^2$. This is the random projection method of \cref{app:projections}. Its risk bound concerns a fresh task, whereas forgetting is evaluated on tasks that determine $x_{k+1}$. \Cref{lem:seen} in \cref{app:aux} relates these two quantities. Combining it with our last-iterate bounds gives the following guarantee.

\begin{corollary}[forgetting under random task orderings]
\label{cor:forgetting}
In the setting above, the expected forgetting of the algorithm in \cref{eq:continual-update} satisfies, for every $k\geq1$,
\begin{align*}
    \E\biggl[\frac1k\sum_{t=1}^k\frac12\operatorname{dist}(x_{k+1},C_{z_t})^2\biggr]
    \leq
    \frac{129\norm{x_1-x^\star}^2\log^2(ek)}k.
\end{align*}
If the sets $C_z$ are affine subspaces, the left-hand side is at most $10\norm{x_1-x^\star}^2/k$.
\end{corollary}

\begin{proof}
Apply \cref{lem:seen} with $\beta=\eta=1$. Since $F(x)-F(x^\star)=\frac12\E_{z\sim\cZ}\operatorname{dist}(x,C_z)^2$ and $x_k$ is independent of $z_k$, \cref{cor:projections} gives $\E[F(x_k)-F(x^\star)]\leq64\norm{x_1-x^\star}^2\log^2(ek)/k$. The first bound follows using $\log^2(ek)\geq1$.

For affine solution sets, the losses are quadratic and their Hessians are orthogonal projections. In this case, \cref{thm:quadratic} gives $\E[F(x_k)-F(x^\star)]\leq9\norm{x_1-x^\star}^2/(2k)$, and \cref{lem:seen} yields the bound $10\norm{x_1-x^\star}^2/k$.
\end{proof}

For continual linear regression with task data $(X_z,y_z)$, the solution set is $C_z=\{x\colon X_zx=y_z\}$. The original task loss satisfies
\begin{align*}
    \frac12\norm{X_zx-y_z}^2
    =\frac12\norm{X_z(x-P_{C_z}x)}^2
    \leq\norm{X_z}^2\cdot\frac12\operatorname{dist}(x,C_z)^2,
\end{align*}
where $\norm{X_z}$ is the operator norm. Thus, if $R\eqq\sup_z\norm{X_z}<\infty$ and $x_1=0$, the expected forgetting in the original losses is at most $10R^2\norm{x^\star}^2/k$. This improves upon the $O(k^{-1/4})$ rate of \citet{evron2025better} and agrees with the $O(1/k)$ rate implied by \citet{derezinski2026}.

\section{Auxiliary lemmas}
\label{app:aux}

In this section, we present the auxiliary lemmas used in our analysis, which are known or adapted from previous work.

\subsection{Average-iterate analysis}

The following bound on the regret of SGD, which is established in the proof of Theorem~8 of \citet{attia2025fast} (by Eq.~(13) there and convexity), for any stepsize $0<\eta<2/\beta$ in the low-noise regime.

\begin{lemma}[{\citealp[Theorem~8]{attia2025fast}}]
\label{lem:energy}
Under the conditions of \cref{thm:main}, for every $T\geq1$,
\begin{align*}
    \sum_{t=1}^T\E[F(x_t)-F(x^\star)]
    \leq
    \frac{2\norm{x_1-x^\star}^2}{\eta(2-\beta\eta)}+\frac{8\eta T}{(2-\beta\eta)^2}\sigma_\star^2.
\end{align*}
\end{lemma}

\subsection{A suffix-averaging inequality}

We use the following inequality, which is Lemma~1 of \citet{orabona2020last} with constant weights (Lemma~1 there is stated for non-increasing weights); it underlies the suffix-averaging technique of \citet{shamir13sgd} (see the proofs of Theorems~1 and~2 therein).

\begin{lemma}[{\citealp[Lemma~1]{orabona2020last}}]
\label{lem:suffix}
For every $T\geq1$ and nonnegative numbers $v_1,\ldots,v_T$,
\begin{align*}
    v_T
    \leq
    \frac1T\sum_{t=1}^Tv_t+\sum_{k=1}^{T-1}\frac1{k(k+1)}\sum_{t=T-k+1}^T(v_t-v_{T-k}).
\end{align*}
\end{lemma}

\subsection{Properties of smooth convex functions}

We use the following standard property of convex and smooth functions given at \citet{nesterov1998introductory}.

\begin{lemma}[{\citealp[Theorem~2.1.5]{nesterov1998introductory}}]
\label{lem:cocoercive}
Under the conditions of \cref{thm:main}, for all $x,y\in\R^d$ and $z\in Z$,
\begin{align*}
    f(x;z)
    \geq
    f(y;z)+\ip{\nabla f(y;z)}{x-y}+\frac1{2\beta}\norm{\nabla f(x;z)-\nabla f(y;z)}^2.
\end{align*}
\end{lemma}

We also use the following self-bounding property of smooth losses, (e.g. \citet{srebro2010smoothness}).

\begin{lemma}[{\citealp{srebro2010smoothness}}]
\label{lem:selfbound}
Under the conditions of \cref{thm:main}, $\E_{z\sim\cZ}\nabla f(x^\star;z)=\nabla F(x^\star)=0$, and for every $x\in\R^d$,
\begin{align*}
    \E_{z\sim\cZ}\norm{\nabla f(x;z)-\nabla f(x^\star;z)}^2
    &\leq
    2\beta\bigl(F(x)-F(x^\star)\bigr),
    \\
    \E_{z\sim\cZ}\norm{\nabla f(x;z)}^2
    &\leq
    4\beta\bigl(F(x)-F(x^\star)\bigr)+2\sigma_\star^2.
\end{align*}
\end{lemma}

The following lemma adapts the stability argument of \citet[Lemma~F.2]{evron2025better} to single-update deletions.

\begin{lemma}[adapted from {\citealp[Lemma~F.2]{evron2025better}}]
\label{lem:seen}
Under the conditions of \cref{thm:main}, assume in addition that $\sigma_\star=0$ and $\eta=1/\beta$. Then, for every $k\geq1$,
\begin{align*}
    \frac1k\sum_{t=1}^k\E\bigl[f_t(x_{k+1})-f_t(x^\star)\bigr]
    \leq
    2\E[F(x_k)-F(x^\star)]+\frac{\beta\norm{x_1-x^\star}^2}k.
\end{align*}
\end{lemma}

\subsection{Properties of quadratic losses}

For quadratic losses in the interpolation regime, we use the following lemma from \cite{evron2025better}.

\begin{lemma}[{\citealp[Lemma~E.3]{evron2025better}}]
\label{lem:quad-identity}
Under the conditions of \cref{thm:quadratic}, for all $x\in\R^d$ and $z\in Z$,
\begin{align*}
    \ip{\nabla f(x;z)}{x-x^\star}
    =
    2\bigl(f(x;z)-f(x^\star;z)\bigr).
\end{align*}
\end{lemma}

\section{Proofs for \texorpdfstring{\cref{sec:quadratic}}{Section 4}}
\label{app:quadratic}

Throughout this section, the assumptions of \cref{thm:quadratic} hold, $0<\eta<2/\beta$, and $x_1,x_2,\ldots$ are the iterates of SGD with stepsize $\eta$; we denote $d_s\eqq\E\norm{x_s-x^\star}^2$ for $s\geq1$.

\subsection{Proof of \texorpdfstring{\cref{lem:quad-decrease}}{Lemma 8}}

\begin{proof}[Proof of \cref{lem:quad-decrease}]
\allowdisplaybreaks
Fix $t\geq1$, and for every $r\in\{1,\ldots,t\}$, let $\tilde x^{(r)}$ be the virtual iterates defined in \cref{eq:restart} with $h=1$, restarted from $x_r$ at time $r+1$. Fix $r\in\{1,\ldots,t\}$. The virtual iterate $\tilde x^{(r)}_{t+1}$ is obtained from $x_1$ by the updates with the samples $z_1,\ldots,z_{r-1},z_{r+1},\ldots,z_t$, whereas $x_t$ is obtained from $x_1$ by the updates with the samples $z_1,\ldots,z_{t-1}$; as the samples are i.i.d., $\tilde x^{(r)}_{t+1}$ is distributed as $x_t$, and thus $\E\norm{\tilde x^{(r)}_{t+1}-x^\star}^2=d_t$. Moreover,
\begin{align*}
    \norm{\tilde x^{(r)}_{t+1}-x^\star}^2
    =
    \norm{x_{t+1}-x^\star}^2+2\ip{x_{t+1}-x^\star}{\tilde x^{(r)}_{t+1}-x_{t+1}}+\norm{\tilde x^{(r)}_{t+1}-x_{t+1}}^2.
\end{align*}
Since the virtual iterates and SGD are updated with the same samples, by non-expensiveness of SGD, it holds that $\norm{\tilde x^{(r)}_{s+1}-x_{s+1}}\leq\norm{\tilde x^{(r)}_s-x_s}$ for every $s\geq r+1$; hence, $\norm{\tilde x^{(r)}_{t+1}-x_{t+1}}\leq\norm{\tilde x^{(r)}_{r+1}-x_{r+1}}=\norm{x_{r+1}-x_r}$. Taking expectations in the last display and averaging over $r=1,\ldots,t$, we obtain
\begin{align*}
    d_t-d_{t+1}
    \leq
    \frac1t\biggl(\sum_{r=1}^t\E\norm{x_{r+1}-x_r}^2+2\E\ip{x_{t+1}-x^\star}S\biggr),
    \qquad\text{where }
    S\eqq\sum_{r=1}^t\bigl(\tilde x^{(r)}_{t+1}-x_{t+1}\bigr).
\end{align*}
We next show that, almost surely,
\begin{align}\label{eq:quad-sum}
    \norm S^2
    \leq
    \frac{\norm{x_1-x^\star}^2-\norm{x_{t+1}-x^\star}^2}{(2-\beta\eta)^2},
\end{align}
which gives the claim: by the inequality $2\ip uw\leq\lambda\norm u^2+\norm w^2/\lambda$ with $\lambda=1/(2-\beta\eta)$, we then have $2\ip{x_{t+1}-x^\star}S\leq\frac{\norm{x_{t+1}-x^\star}^2}{2-\beta\eta}+(2-\beta\eta)\norm S^2\leq\frac{\norm{x_1-x^\star}^2}{2-\beta\eta}$.

To prove \cref{eq:quad-sum}, we use the linearity of the updates. For $s\geq1$, let $B_s\eqq I-\eta H_{z_s}$, which is a symmetric matrix with $(1-\beta\eta)I\preceq B_s\preceq I$. Since $\nabla f(x;z)=H_z(x-x^\star)$, the update of SGD reads $x_{s+1}-x^\star=B_s(x_s-x^\star)$, and similarly, $\tilde x^{(r)}_{s+1}-x^\star=B_s(\tilde x^{(r)}_s-x^\star)$ for $s\geq r+1$. As $\tilde x^{(r)}_{r+1}-x_{r+1}=x_r-x_{r+1}=(I-B_r)(x_r-x^\star)$, we obtain
\begin{align*}
    \tilde x^{(r)}_{t+1}-x_{t+1}
    =
    B_t\cdots B_{r+1}(I-B_r)(x_r-x^\star),
\end{align*}
where the product $B_t\cdots B_{r+1}$ is the identity for $r=t$. Moreover, $I-B_s\preceq\frac1{2-\beta\eta}(I-B_s^2)$ for every $s$, since every eigenvalue $\lambda\in[1-\beta\eta,1]$ of $B_s$ satisfies $1+\lambda\geq2-\beta\eta$, and thus $1-\lambda\leq\frac{(1-\lambda)(1+\lambda)}{2-\beta\eta}$. As $\norm{x_s-x^\star}^2-\norm{x_{s+1}-x^\star}^2=(x_s-x^\star)^\top(I-B_s^2)(x_s-x^\star)$ for every $s\geq1$, summing over $s=1,\ldots,t$ gives
\begin{align*}
    \sum_{s=1}^t(x_s-x^\star)^\top(I-B_s)(x_s-x^\star)
    \leq
    \frac{\norm{x_1-x^\star}^2-\norm{x_{t+1}-x^\star}^2}{2-\beta\eta}.
\end{align*}
We pair $S$ with a vector propagated backward by the same matrices. Fix $v\in\R^d$, let $y_{t+1}\eqq v$, and let $y_s\eqq B_sy_{s+1}$ for $s=t,t-1,\ldots,1$. As the matrices $B_s$ are symmetric, $y_{r+1}=B_{r+1}\cdots B_tv=(B_t\cdots B_{r+1})^\top v$, and thus $\ip v{\tilde x^{(r)}_{t+1}-x_{t+1}}=\ip{y_{r+1}}{(I-B_r)(x_r-x^\star)}$. Since $I-B_r\succeq0$ for every $r$, the Cauchy--Schwarz inequality for the positive semidefinite bilinear form $((u_r)_r,(w_r)_r)\mapsto\sum_{r=1}^tu_r^\top(I-B_r)w_r$ gives
\begin{align*}
    \ip vS^2
    &=
    \biggl(\sum_{r=1}^t\ip{y_{r+1}}{(I-B_r)(x_r-x^\star)}\biggr)^2
    \\&\leq
    \sum_{r=1}^ty_{r+1}^\top(I-B_r)y_{r+1}\cdot\sum_{r=1}^t(x_r-x^\star)^\top(I-B_r)(x_r-x^\star).
\end{align*}
By the last display, the second sum is at most $(\norm{x_1-x^\star}^2-\norm{x_{t+1}-x^\star}^2)/(2-\beta\eta)$, and  Since $I-B_r\preceq\frac1{2-\beta\eta}(I-B_r^2)$ and $y_{r+1}^\top(I-B_r^2)y_{r+1}=\norm{y_{r+1}}^2-\norm{y_r}^2$, the first sum is at most
\begin{align*}
    \frac1{2-\beta\eta}\sum_{r=1}^t\bigl(\norm{y_{r+1}}^2-\norm{y_r}^2\bigr)
    =
    \frac{\norm v^2-\norm{y_1}^2}{2-\beta\eta}
    \leq
    \frac{\norm v^2}{2-\beta\eta}.
\end{align*}
Hence, $\ip vS^2\leq\norm v^2(\norm{x_1-x^\star}^2-\norm{x_{t+1}-x^\star}^2)/(2-\beta\eta)^2$, and taking $v=S$ gives \cref{eq:quad-sum}.
\end{proof}

\subsection{Proof of \texorpdfstring{\cref{lem:quad-reduction}}{Lemma 9}}

\begin{proof}[Proof of \cref{lem:quad-reduction}]
Fix $t\geq1$. By \cref{lem:quad-identity}, and  Since $x_t$ is independent of $z_t$, we have $\E\ip{\nabla f_t(x_t)}{x_t-x^\star}=2\E[F(x_t)-F(x^\star)]$, thus,
\begin{align*}
    d_{t+1}
    =
    \E\norm{x_t-\eta\nabla f_t(x_t)-x^\star}^2
    =
    d_t-4\eta\E[F(x_t)-F(x^\star)]+\eta^2\sigma_t^2.
\end{align*}
Moreover, Since $\nabla f(x^\star;z)=0$ for every $z\in Z$, applying \cref{lem:selfbound} conditionally on $z_1,\ldots,z_{t-1}$ gives $\sigma_t^2\leq2\beta\E[F(x_t)-F(x^\star)]$. Combining the two, $d_t-d_{t+1}\geq\bigl(\frac{2\eta}\beta-\eta^2\bigr)\sigma_t^2=\frac{\eta(2-\beta\eta)}\beta\sigma_t^2$, and the claim follows from \cref{lem:quad-decrease}.
\end{proof}

\subsection{Proof of \texorpdfstring{\cref{lem:quad-conversion}}{Lemma 10}}

\begin{proof}[Proof of \cref{lem:quad-conversion}]
Fix $t\geq1$. By \cref{lem:quad-identity}, and  Since $x_t$ is independent of $z_t$, we have $\E\ip{\nabla f_t(x_t)}{x_t-x^\star}=2\E[F(x_t)-F(x^\star)]$, thus,
\begin{align*}
    d_{t+1}
    =
    \E\norm{x_t-\eta\nabla f_t(x_t)-x^\star}^2
    =
    d_t-4\eta\E[F(x_t)-F(x^\star)]+\eta^2\sigma_t^2,
\end{align*}
that is, $\E[F(x_t)-F(x^\star)]=\frac1{4\eta}(d_t-d_{t+1})+\frac\eta4\sigma_t^2$, and the claim follows from \cref{lem:quad-decrease}.
\end{proof}

\end{document}